\documentclass{article}
\let\labelindent\relax
\usepackage{amsfonts}       

\usepackage{amsthm}
\usepackage{mathtools}      
\usepackage{amssymb}        
\usepackage{filecontents}
\usepackage{graphicx}       
\usepackage{marginnote}     
\usepackage{marvosym}       
\usepackage{overpic}        
\usepackage{tabularx}
\usepackage{cite}
\usepackage{color}
\usepackage[linesnumbered,algoruled,boxed,lined]{algorithm2e}
\usepackage[normalem]{ulem}
\usepackage{epstopdf}
\usepackage{enumitem}
\usepackage[normalem]{ulem}

\usepackage{MnSymbol}

\usepackage[paperheight=11in,paperwidth=8.5in,
			left=1in,right=1in,
			top=1in,bottom=1in]{geometry}

\newtheorem{prob}{Problem}[section]
\newtheorem{lemma}{Lemma}[section]
\newtheorem{proposition}{Proposition}[section]

\newtheorem{corollary}{Corollary}[section]
\newtheorem{theorem}{Theorem}[section]
\theoremstyle{definition}
\newtheorem{definition}{Definition}[section]

\newif\ifarxiv
\arxivtrue

\usepackage{mathtools,xparse}

\def\lsr{\texttt{LSR}\xspace}
\def\usr{\texttt{CSR}\xspace}
\def\mrmp{\texttt{MRMP}\xspace}
\def\sag{\textsc{SaG}\xspace}
\def\paf{\textsc{PaF}\xspace}

\begin{document}
\title{
Rubik Tables and Object Rearrangement
}
\author{
Mario Szegedy and Jingjin Yu
\thanks{
M. Szegedy and J. Yu are with the Department of Computer Science, 
Rutgers, the State University of New Jersey, Piscataway, NJ, 
USA. E-Mails: \{{\tt mario.szegedy,jingjin.yu}\}\hspace*{.25em}
\MVAt \hspace*{.25em}rutgers.edu. 
}%
}
\maketitle

\begin{abstract}
A great number of robotics applications demand the rearrangement of 
many mobile objects, e.g., organizing products on store shelves, shuffling 
containers at shipping ports, reconfiguring fleets of mobile robots, 
and so on. 
To boost the efficiency/throughput in systems designed for solving these 
rearrangement problems, it is essential to minimize the number of atomic 
operations that are involved, e.g., the pick-n-places of individual objects. 
However, this optimization task poses a rather difficult challenge due to 
the complex inter-dependency between the objects, especially when they are 
tightly packed together. \vspace{1mm}

In this work, in tackling the aforementioned challenges, we have developed 
a novel algorithmic tool, called Rubik Tables, that provides a clean 
abstraction of object rearrangement problems as the proxy problem of 
shuffling items stored in a table or lattice. 
In its basic form, a Rubik Table is an $n\times n$ table containing $n^2$ 
items. 
We show that the reconfiguration of items in such a Rubik Table can be 
achieved using at most $n$ column and $n$ row shuffles in the partially labeled 
setting, where each column (resp., row) shuffle may arbitrarily permute the 
items stored in a column (resp., row) of the table. 
When items are fully distinguishable, additional $n$ shuffles are needed. 
Rubik Tables allow many generalizations, e.g., adding an additional depth 
dimension or extending to higher dimensions.  
\vspace{1mm}

Using Rubik Table results, we have designed a first constant-factor optimal 
algorithm for stack rearrangement problems where items are stored in stacks, 
accessible only from the top. 
We show that, for $nd$ items stored in $n$ stacks of depth $d$ each, using one 
empty stack as the swap space, $O(nd)$ stack pop-push operations are sufficient for 
an arbitrary reconfiguration of the stacks where $d \le n^{\frac{m}{2}}$ for 
arbitrary fixed $m >0$. 
Rubik Table results also allow the development of constant-factor optimal 
solutions for solving multi-robot motion planning problems under extreme robot 
density. These algorithms based on Rubik Table results run in low-polynomial 
time.  


\end{abstract}

\section{Introduction}
In a broad range of real-world applications, items are arranged in 
\emph{stacks} to balance between efficient space usage and the ease of 
storage and retrieval (see Fig.~\ref{figure:stack-ex}). 
In a stack based storage solution, only the item on the top of an non-empty 
stack can be accessed instantaneously. If other stored items are to be 
retrieved, additional items must be moved beforehand. 
Such an approach, while preventing the direct random access of an arbitrary 
item, allows more economical utilization of the associated storage space, 
which is always limited. 
A prime example is the stacking of containers at shipping ports 
\cite{borgman2010online,dayama2014approaches}, where stacks of containers 
may need to be rearranged (shuffled) for retrieval in a specific order, 
which contains a stack rearrangement component. 
Similar scenarios also appear frequently elsewhere, e.g., parking yards 
during busy hours in New York City, the re-ordering of misplaced grocery 
items on supermarket shelves \cite{HanStiBekYu2017}, the rearrangement of 
goods in warehouses \cite{christofides1973rearrangement}, and so on. In 
all these application scenarios, the overall efficiency of the system 
critically depends on minimizing the number of item storage and retrieval 
operations. However, the tightly packed items, combined with the 
stack-based access, induce complex dependencies that make the 
rearrangement tasks a challenge to optimize. 

\begin{figure}[ht!]
	\begin{center}
  \begin{overpic}[width={\ifarxiv 0.8\columnwidth \else 0.95\columnwidth \fi}]{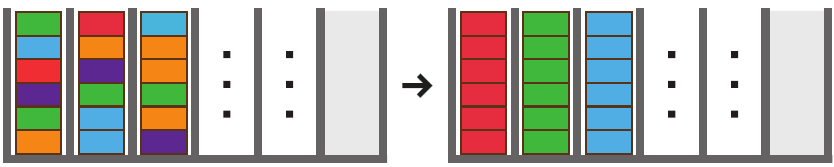}
	\put(3.5,-3){{\small $1$}}
	\put(11,-3){{\small $2$}}
	\put(18.4,-3){{\small $3$}}
	\put(25,-3){{\small $\ldots$}}
	\put(33.7,-3){{\small $n$}}
	\put(38,-3){{\small $n+1$}}
	\put(-3,15.6){{\small $1$}}
	\put(-2.5,8){\rotatebox{90}{{\small $\ldots$}}}
	\put(-3,1.8){{\small $d$}}
	\put(41.2,2.2){\rotatebox{90}{{\small buffer\,stack}}}
	\end{overpic}
	\end{center}
  \caption{\label{figure:stack-ex} An illustration of the colored stack 
	rearrangement problem with an initially empty buffer stack. From an initial 
	arrangement of the items shown on the left, we would like to sort the items 
	into the ordered target arrangement, by item color or type, on the right. 
	Items may assume unique identifiers in the labeled version of the problem.}
\end{figure}

Similarly, in many large-scale multi-robot applications, e.g., warehouse 
automation \cite{wurman2008coordinating,li2020lifelong}, it is desirable 
to operate as many robots as possible in a confined workspace: a small 
workspace requires less travel time for robots between tasks and more 
robots can execute more tasks simultaneously. At the same time, the 
efficient routing of a large number of robots in limited space again 
requires minimizing redundant movement of objects, in this case the 
robots. Such problems are well-known to be computationally intractable 
\cite{yu2013structure,solovey2016hardness} due to the curse of 
dimensionality. As such, whereas tractable solutions have been proposed 
at least a few decades ago \cite{kornhauser1984coordinating}, polynomial
time algorithms with good optimality guarantees have been illusive until 
recently, e.g., \cite{Yu18RSS}, due partly to this study.  

Motivated by the array of practical scenarios, we have performed a 
systematic study of multi-object rearrangement problems where the objects 
are either movable items or robots that could move on their own. As we 
pay close attention to the structure of these apparently diverse set of 
robotics problems (e.g., rearranging items in stacks and multi-robot motion 
planning), we observe that, as the density of the objects become sufficiently 
high, they can be effectively viewed as a problem of \emph{shuffling items 
stored in tables or lattices}. This leads us to the development of the 
\emph{Rubik Table} abstraction (see Fig.~\ref{figure:rubik-ex}). In a basic
two-dimensional setting, a Rubik Table is an $n\times n$ table storing 
$n^2$ items. These items may either have $n$ colors (or types) with each color 
containing $n$ items, or they may all have unique labels. The \emph{Rubik 
Table problem} asks for the sorting of the items by type or by label using 
the least number of column and row shuffles, where a column (resp., row) 
shuffle allows the arbitrary permutation of the items in a column (resp., 
row). 
Intuitively, solving the color-based Rubik Table sorting problem appears 
to require at least about $2n$ shuffles; it seems that each row and column 
should be worked on at least once. Surprisingly, the reconfiguration of the 
Rubik Table by color can be achieved using at most $2n$ column and row 
shuffles, through a careful application of a version of the Hall's marriage
theorem \cite{hall2009representatives}. As a consequence, sorting by label
is achievable using $3n$ shuffles.  

\begin{figure}[h]
	\begin{center}
  \begin{overpic}[width={\ifarxiv 0.7\columnwidth \else \columnwidth \fi}]{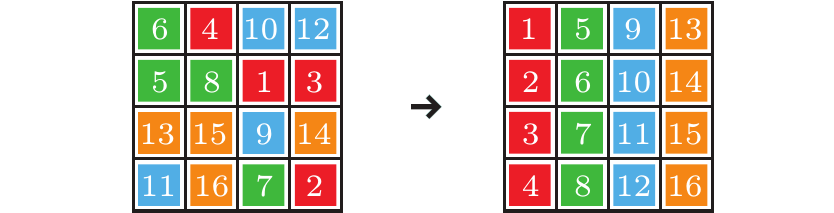}
	\put(27,-4){{\small (a)}}
	\put(71,-4){{\small (b)}}
	\end{overpic}
	\end{center}
  \caption{\label{figure:rubik-ex} An instance of a $4\times 4$ Rubik 
	Table problem where we want to rearrange the items to go from the left 
	configuration	to the right configuration, sorted by color or by label.}
\end{figure}

The Rubik Table results, with the extensions, turn out to be fairly powerful 
in helping boost the asymptotic efficiency of algorithms for solving object 
rearrangement tasks. In the stack rearrangement problem 
(Fig.~\ref{figure:stack-ex}), $nd$ items are stored in $n$ stacks of 
capacity $d$ each. With a buffer stack of depth $d$, we are to sort the 
items by type ($n$ of these, $d$ each) or by their unique labels using 
stack pop-push operations. We use \usr (resp., \lsr) as shorthand for the
colored (resp., labeled) stack rearrangement problem. 
For \usr and \lsr, through simulating column and row shuffles using a 
series of specially crafted stack pop-push operations, it is possible to 
apply a version of the Rubik Table results to derive an algorithm that 
requires only $O(nd)$ number of pop-pushes when $d \le n^{\frac{m}{2}}$ for 
an arbitrary
fixed $m > 0$. This suggests that the polynomial-time algorithm computes 
$O(1)$-optimal solutions for these settings. Previously, the best upper 
bound \cite{HanStiBekYu2017} incurs an additional logarithmic factor. 

In a similar vein, but to an apparently different problem, the Rubik Table 
theorems also apply to improving polynomial time multi-robot 
motion planning (\mrmp) algorithms. 
In the context of the current work, \emph{motion} refers to discrete time 
synchronous motion of many robots in a grid-like environment. 
To our knowledge, no polynomial time algorithms exist prior to 
\cite{Yu18RSS,demaine2019coordinated} that provide guaranteed solution makespan 
optimality for \mrmp. Through simulating column/row shuffles using more 
rudimentary multi-robot motion primitives, a version of the Rubik Table 
result gives rise to an expected $O(1)$-optimal solution for arbitrary 
fixed dimensions, computed in low polynomial time \cite{Yu18RSS}. Using 
this expected $O(1)$-optimal solution, an $O(1)$-optimal solution can be 
constructed as well. In this study, we highlight the application of Rubik
Table to \mrmp and describe a new and improved expected $O(1)$-optimal 
algorithm for \mrmp. 

In the robotics domain, our study relates to multi-object 
rearrangement tasks, which may be carried out using mobile robots 
\cite{ben1998practical,garrett2015ffrob,havur2014geometric} or fixed robot 
arms \cite{krontiris2015dealing,krontiris2016efficiently,
han2018complexity,huang2019large}. Clearly 
a challenging task and motion planning (TAMP) problem in the general setting 
\cite{huang2019large}, even the combinatorial aspect of object rearrangement 
is shown to be computationally hard in multiple problems in seemingly simple 
setups \cite{han2018complexity}. A multi-arm rearrangement problem is 
recently explored \cite{shome2018fast}. 
In a more abstract setting, multi-object rearrangement has also been 
studied under the PushPush line of problems \cite{demaine2001pushing,
demaine2000pushpush}.
More broadly, object rearrangement problems are connected to multi-robot 
motion planning problems \cite{erdmann1987multiple,kornhauser1984coordinating,
solovey2016hardness,Yu18RSS,demaine2019coordinated} and the problem of navigation among movable 
obstacles \cite{wilfong1991motion,stilman2008planning,van2009path}. 
Lastly, as a sorting problem, our study shares some similarities with 
sorting networks \cite{ajtai1983,west1993sorting}. 
Indeed, results on sorting networks can be applied to solve object 
rearrangement problems. 

The stack rearrangement problem, which we study to some depth in this work, 
was first formally studied in the stated form in \cite{HanStiBekYu2017}. 
They established an $O(nd(\log n + \log d))$ algorithmic upper 
bound. Heuristics-based search methods are also developed that can 
compute the optimal solution for stack rearrangement problems involving 
tens of items. 
A closely related problem is the Hanoi tower problem 
\cite{brousseau1980tower,szegedy1999many,grigorchuk2006asymptotic}, which 
has additional constraints limiting the relative order of items in a stack 
during the rearrangement process. 

The main algorithmic contributions of this work, beside building a novel 
structural connection between the abstract class of Rubik Table problems and 
challenging object rearrangement problems, are: 
\begin{itemize}
\item For an $n \times n$ Rubik Table, sorting by the $n$ types can be 
achieved using at most $2n$ column and/or row shuffles; if the items are all 
uniquely labeled, sorting can be done using at most $3n$ shuffles 
(Theorem~\ref{t:rubik-table}). The algorithm for computing the required 
shuffle operations takes $O(n^2\log n)$ expected time or $O(n^3)$ 
deterministic time. The Rubik Table results generalize to including a depth 
dimension (Theorem~\ref{t:fat-rt}) and to higher dimensions (Theorem~\ref{t:rrd}). 
\item For any fixed $m > 0$, \lsr (and therefore, \usr) with $d \le n^{\frac{m}{2}}$ 
can be solved using $O(nd)$ pop-push operations. If $m$ is an input parameter instead, 
\lsr with $d = n^{\frac{m}{2}}$ can be solved using $O(3^mnd)$ pop-pushes 
(Theorem~\ref{t:lsr-d-ge-n}). Therefore, for an arbitrary fixed real number 
$c$, \lsr may be solved using $O(nd)$ pop-pushes for $d \le \lceil cn \rceil$ 
(Theorem~\ref{t:lsr-d-cn}).
\item The \mrmp problem on an $m_1\times m_2$ grid can be solved using 
a makespane of $O(m_1 + m_2)$ through a single application of the Rubik 
Table shuffle algorithm. 
\end{itemize}

This paper builds on the conference publication \cite{SzeYu2020WAFR}. 
As we continue to develop the related research, this archive version adapts a 
new structure starting with the Rubik Table results and then develops their 
applications on stack rearrangement and multi-robot motion planning, which 
we believe is a more proper presentation of the research. In synchronization 
with the effort, we have developed several new results including a lower bound on 
the required number of shuffles needed for solving Rubik Table problems 
(Proposition~\ref{p:rubik-lb}), a new upper bound for the stack rearrangement 
problem with the constant in the big $O$ notation given explicitly 
(Proposition~\ref{p:7l} and Corollary~\ref{c:7l}), and a new expected 
$O(1)$ makespan-optimal algorithm for multi-robot motion planning
(Proposition~\ref{p:mrmpn}). We also expanded the discussion of open
problems, which provides several concrete future directions for 
readers to explore. 

The rest of the paper is organized as follows. In Sec.~\ref{sec:prelim}, 
Rubik Table problems, stack rearrangement problems, and a version of the 
multi-robot motion planning problem are formally defined; straightforward 
lower bounds for the Rubik Table problem and the stack rearrangement problem 
are also provided.
In Sec.~\ref{sec:rubik}, we present results on several Rubik Table
problems.
In Sec.~\ref{sec:stack}, refined upper bounds are established for \usr and 
\lsr, demonstrating the utility of the Rubik Table abstraction. 
We further expose the application of Rubik Table to multi-robot motion 
planning in Sec.~\ref{sec:mrmp}. 
We conclude in Sec.~\ref{sec:conclusion} with a brief discussion of many 
interesting open questions.

\section{Preliminaries}\label{sec:prelim}
\subsection{Rubik Table Problems}
The \emph{Rubik Table problem} (see, e.g., Fig.~\ref{figure:rubik-ex}) 
formalizes the task of carrying out globally 
coordinated token swapping operations in lattices. We associate it with the 
name \emph{Rubik} as it shares some similarity with the Rubik's Cube toy. 
The basic setting deals with a planar table. 

\begin{prob}[(Rubik Table problem)]\label{prob:rt}
Let $M$ be an $n \times n$ table containing $n^2$ items, one in each 
table cell. The $n^2$ items are of $n$ types with each type having a 
multiplicity of $n$. In a shuffle operation, the items in a single column 
or a single row of $M$ may be permuted in an arbitrary manner. Given an 
arbitrary configuration $X_I$ of the items, find a sequence of shuffles 
that take $M$ from $X_I$ to the configuration where column $i$, $1 \le i
\le n$, contains only items of type $i$. 
\end{prob}

Intuitively, $\Omega(n)$ shuffles are required for solving the Rubik 
Table problem, assuming that items are randomly distributed. 
We formally establish this lower bound here. 

\begin{proposition}[(Lower bound for the Rubik Table problem)]
\label{p:rubik-lb}
A random Rubik Table problem instance requires at least $n-1$ shuffles 
to solve, in expectation. 
\end{proposition}
\begin{proof} 
We may consider the probability that the content of cell $(i, i)$ must 
participate in a shuffle (of either column $i$ or row $i$). The 
probability of cell $(i, i)$ containing items with type $i$ is $\frac{1}{n}$, 
suggesting that the probability that the cell must be shuffled is 
$\frac{n-1}{n}$. The expected number of shuffles required to place an 
item of type $i$ at cell $(i, i)$ is then at least $\frac{n-1}{n}$. 
Summing over all $i$, $1 \le i \le n$, and by the linearity of expectation,
yields that at least $n-1$ shuffles are required to solve the Rubik 
Table problem, in expectation. 
\qed
\end{proof}

One may further require that the $n^2$ items be uniquely labeled $1, \ldots, n^2$, 
and ask for an arbitrary reconfiguration of the items. This is equivalent 
to sorting the items by label, i.e., going from Fig.~\ref{figure:rubik-ex}(a)
to Fig.~\ref{figure:rubik-ex}(b), ignoring the colors. We call this the 
\emph{labeled Rubik Table problem}. 

\begin{prob}[(Labeled Rubik Table problem)]\label{prob:rt}
Let the $n^2$ items in an $n\times n$ Rubik Table $M$ have unique labels 
from $1, \ldots, n^2$. Given an arbitrary configuration $X_I$ of the items, 
find a sequence of shuffles that takes $M$ from $X_I$ to the column-major 
sequential ordering of the $n^2$ items. 
\end{prob}

We may allow a Rubik Table $M$ to have a ``depth'' $K$. In each column 
(resp., row) shuffle of such a ``fat'' Rubik Table, we allow the 
arbitrary permutation of a ``fat'' column (resp., row) which contains $nK$ 
items. 

\begin{prob}[(Fat Rubik Table problem)]\label{prob:frt}
Let $M$ be an $n \times n \times K$ (row $\times$ column $\times$ depth) 
table containing $n^2K$ items, one in each cell of the table. These 
items are of $n$ types, with each type having a multiplicity of $nK$. 
In a shuffle operation, the items in a single fat column (i.e., items 
with indices in $\{1, \ldots, n\} \times \{i\} \times \{1, \ldots, 
K\}$ for $1 \le i \le n$) or a single fat row (i.e., items with 
indices in $\{i\} \times \{1, \ldots, n\} \times \{1, \ldots, K\}$ for 
$1 \le i \le n$) of $M$ may be permuted in an arbitrary manner. Given 
a configuration $X_I$ of the items, find a sequence of shuffles that 
takes the table $M$ from $X_I$ to the configuration where each fat 
column $i$, $1 \le i \le n$, only contains items of type $i$. 
\end{prob}

Similarly, a labeled version of the fat Rubik Table problem may be 
defined. We omit the straightforward definition. 

The Rubik Table problem may be further generalized to arbitrary 
dimensions. 

\begin{prob}[(Rubik $R$-D Table problem)]\label{prob:rrd}Let $M$ be an 
$\underbrace{n\times \ldots \times n}_R$ table, $R \ge 2$, filled with 
$n^R$ unique items. Assuming that any $(R-1)$-dimensional column can 
be arbitrarily shuffled, given two arbitrary configurations of the 
items, $X_I$ and $X_G$, find a sequence of shuffles that takes $M$ 
from $X_I$ to $X_G$. 
\end{prob}

\subsection{Stack Rearrangement}
In a \emph{stack rearrangement} problem (see, e.g., 
Fig.~\ref{figure:stack-ex}), there are $n$ stacks (i.e., LIFO queues), 
each filled to capacity with $d$ items. In the \emph{labeled} version, 
or \lsr (labeled stack rearrangement), the items in the stacks are 
uniquely labeled $1, \ldots, nd$. Given an arbitrary initial arrangement 
of the items, we would like to rearrange them to follow lexicographic 
order, in which the $k^{\rm th}$ stack, $1 \le k \le n$, contains items 
labeled $(k-1)d + 1$ to $kd$, with numbers decreasing (or increasing) 
monotonically from the top of the stack to the bottom of the stack. 
In a single \emph{pop-push} stack operation, an item can be popped 
off from any non-empty stack and immediately pushed onto a stack which is 
not filled to its capacity $d$. To allow the rearrangement of items, we 
assume that there is an empty buffer stack with capacity $d$. During the 
moves the buffer can hold items but it must be emptied by the end. We  
seek to \emph{minimize} the number of pop-pushes to take the stacks 
from an arbitrary initial arrangement to the specified target arrangement,
which is equivalent to having an arbitrary goal arrangement. 

In a \emph{colored} version, or \usr (colored stack rearrangement), we 
still require that items labeled $(k-1)d +1, \ldots, kd$ go into the 
$k^{th}$ stack but do not require these items take a specific order within 
the stack. This is equivalent to saying that we would like to sort $nd$ 
items with $n$ types of $d$ each so that the $k^{th}$ stack contains only 
items of type $k$. 

It takes at least $\Omega(nd)$ pop-pushes to solve the stack rearrangement 
problem for a typical input instance, because most items must move at 
least once to get into place. Here, we prove a stronger lower 
bound. We mention that similar bounds are described in 
\cite{HanStiBekYu2017}. We provide a more accurate bound for \lsr here 
with a proof counting the number of bits required to describe an 
algorithm. A bound for \usr is also included for completeness. 
 
\begin{lemma}[(Lower bound for \lsr)]\label{l:lsr-lb}
An average \lsr instance requires $\Omega(nd + nd{\frac{\log d}{\log 
n}} )$ pop-pushes to resolve.
\end{lemma}
\begin{proof}
The proof is by a counting argument. Any correct algorithm must follow 
different paths for all of the $(nd)!$ initial arrangements, since two 
different initial arrangements followed by identical moves would lead 
to different final arrangements. A step (corresponding to one pop-push
operation) of the algorithm can be 
described with $2 \lceil \log (n+1)\rceil$ bits: (from where, to where). 
Therefore, the two-based logarithm of the number of possible sequences 
of at most $t$ steps is upper bounded by $2t \lceil \log (n+1)\rceil$. 
So as long as it holds that
\[
2t \lceil \log (n+1)\rceil \le \log \left(0.01\cdot (nd)!\right)  =  \Omega(nd  \log nd),
\]
i.e. when $t = o(nd + nd\frac{\log d}{log n} )$, the initial arrangements 
that can be solved with $t$ steps constitute only a small minority of all 
arrangements. The counter-positive of this gives the lemma.
\end{proof}

\begin{lemma}[(Lower bound for \usr)]\label{l:usr-lb}
Any algorithm for \usr must take at least $\Omega(nd)$ pop-pushes for an average 
input.
\end{lemma}
\begin{proof}
Me may view the generation of a random instance as selecting from $n$ 
types of items with replacement $d$ for up to $nd$ rounds. Therefore, 
there are $(\Theta(n))^{\Theta(nd)}$ initial configurations. Following 
the same argument from the proof of Lemma~\ref{l:lsr-lb}, $\Omega(nd)$ 
pop-pushes are necessary. 
\end{proof}

\subsection{Multi-Robot Motion Planning}\label{subsec:mrmp-d}
For multi-robot motion planning, we focus on a grid-based setting. Let 
$m_i$, $1 \le i \le R$ be the span of the $i^{th}$ dimension of an 
$R$-dimensional grid. That is, we work with an $m_1 \times \ldots 
\times m_R$ grid. There are $\prod_im_i$ robots labeled $1, \ldots, 
\prod_im_i$, each occupying a vertex of the grid. As to the motion 
primitive, in each time step, robots on any set of mutually disjoint, 
non-self-intersecting cycles of the grid may synchronously move across
one edge on these cycles in the same direction, i.e., robots may 
rotate along these cycles (see e.g., Fig.~\ref{figure:mrmp-move}). 

\begin{figure}[h]
	\begin{center}
  \begin{overpic}[width={\ifarxiv 0.7\columnwidth \else \columnwidth \fi}]{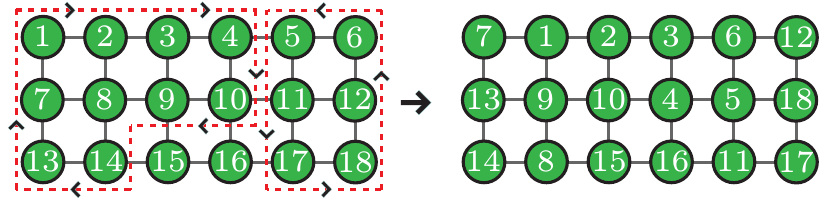}
	\end{overpic}
	\end{center}
  \caption{\label{figure:mrmp-move} Illustration of a single time step move
	of multiple robots along two disjoint cycles of a $3 \times 6$ grid.}
\end{figure}

The makespan-optimal multi-robot motion planning (\mrmp) problem seeks 
a routing plan that takes the robots between two arbitrary configurations
$X_I$ and $X_G$ with the minimum makespan. This is an NP-hard problem 
\cite{yu2015intractability,banfi2017intractability}. In this study, we 
will illustrate how Rubik Table facilitates solving \mrmp to $O(1)$ 
makespan-optimal.

\section{Algorithms for Rubik Table Problems}\label{sec:rubik}
Given that no less than $n-1$ shuffles are required for solving the Rubik 
Table problem in expectation, it is somewhat surprising that only $2n$ 
shuffles can get the job done. 

\begin{theorem}[(Linear shuffle algorithm for Rubik Table 
and labeled Rubik Table problems)]\label{t:rubik-table}
A Rubik Table problem is solvable using $n$ column shuffles followed 
by $n$ row shuffles. Additional $n$ column shuffles then solve the 
labeled Rubik Table problem. 
\end{theorem}

We call the result a linear shuffle algorithm because each item may be 
moved a (small) constant number of times, even though the number of 
column and row shuffles is the square root of the number of items which
is sub-linear.
Before presenting the proof of Theorem~\ref{t:rubik-table}, we introduce 
a K\H{o}nig-Hall type matching theorem \cite{hall2009representatives}
with parallel edges. 

\begin{lemma}[(Hall's Matching theorem with parallel edges)]\label{l:p-hall}
Let $B$ be a $d$-regular ($d > 0$) bipartite graph on $n + n$ nodes, possibly 
with parallel edges. Then $B$ has a perfect matching.
\end{lemma}
\begin{proof}
Let the vertex set of $B$ be $L \bigcupdot R$, where $L$ is the left partite 
set of $B$ and $R$ is the right partite set; the $\bigcupdot$ symbol denotes 
the disjoint union of two sets. Consider a maximal matching $M$ 
in $B$. We show that $M$ meets all of the vertices of $B$, so it is perfect. 
Assume that $M$ is not incident to some vertex $v \in L$. Consider all nodes of 
$B$ reachable by an alternating path from $v$, that is, a path that starts in $v$, 
goes to $R$ along some edge, then goes back to $L$ along an edge of $M$ (if such 
an edge exists), then goes along an arbitrary edge to $R$, an so on, always 
alternating between edges of $M$ and non-edges of $M$. We stop whenever we want.
If any such path $P$ ends up in a point $w$ of $R$ not matched in $M$, we could 
make $M$ bigger as follows: We discard from $M$ its intersection with $P$ and 
add $P \backslash M$ to it, creating 
$$M' = M\triangle P = (M \backslash (M \cap P)) \cup (P \backslash M).$$

It is easy to see that $M'$ is a matching and $|M'| = |M| + 1$, contradicting 
the maximality of $M$. Otherwise, let $L' \subseteq L$ be the subset of $L$ 
reachable via an alternating path from $v$ (which includes $v$ too), and let 
$R' \subseteq R$ be the set of nodes reachable with alternating path from $v$. 
Then $|L'| > |R'|$, since every vertex in $R'$ has a matching partner in $L'$ 
through $M$, and in addition $L'$ contains $v$, which is not a partner of any 
node in $R'$. Furthermore, all neighbors of the nodes in $L'$ must be in $R'$, 
otherwise we could find a neighbor $w$ of a node $t$ in $L'$, which is reachable
via an alternating path from $v$ (formed by adding edge $(t,w)$ to the 
alternating path from $v$ to $t$), but unmatched in $M$. Since the nodes in $L'$
have a total of $d|L'|$ edges incident to them (counted with multiplicities), 
which is more than the number $d|R'|$ of edges incident to $R'$ (counted
with multiplicities), we again have a contradiction.
\end{proof}

\begin{proof}[Proof of Theorem~\ref{t:rubik-table}]
The $n + n + n$ shuffles to construct an arbitrary permutation $\pi$ for solving 
the labeled Rubik Table problem are outlined in Table~\ref{t:shffle}. 
\begin{table}[ht]
\begin{tabularx}{\columnwidth}{l l X}
\hline
1. & Preparation: & By appropriately permuting the items within 
each column we reach the situation where the $n$ items destined to go to 
any fixed column will end up in $n$ different rows. \\ \hline
2. & Column fitting: & By appropriately permuting the items within 
each row we reach the situation where the $n$ items destined to go to 
any fixed column goes to that column.\\ \hline
3. & Row fitting: & By appropriately permuting the items within each 
column we move each item into its final destination. \\ \hline
\end{tabularx}
\vspace*{2mm}
\caption{\label{t:shffle}A three-phase shuffle plan for rearranging labeled 
items in an $n \times n$ Rubik Table. The first $n + n$ shuffles solve the
type/color based Rubik Table problem.}
\end{table}

The preparation phase is necessary for the column fitting phase. We need 
to prove that we can permute the items only within every column (i.e. 
such that no item changes its column coordinate) with the effect that the $n$ 
items destined to go to any fixed column end up in $n$ different rows. 
This comes from Lemma~\ref{l:column-permute}, which shows the feasibility 
of the preparation phase and therefore, the entire algorithm. 
\end{proof}

\begin{lemma}\label{l:column-permute}
Let $M$ be an $n \times n$ table filled with items of $n$ different 
types. The number of items of type $i$ is exactly $n$ for $1 \le i \le n$. 
Then we can permute the items within each column of $M$ separately such that 
in the resulting new arrangement all of the $n$ items of any fixed type
$i$ (for $1 \le i \le n$) go into separate rows. 
\end{lemma}
\begin{proof}
We begin by creating a bipartite graph $B(T,C)$ on $n + n$ nodes such that 
the left partite set, $T$, stands for all the types $\{1, \ldots n\}$, and 
the right partite set, $C$, stands for all the columns of $M$ (also, see 
Fig.~\ref{figure:rubik}). We draw $k$ edges between type $j$ and column $i$, 
if column $i$ contains $k$ items of type $j$. Notice that $B$ is $n$-regular 
from both sides with parallel edges. Lemma~\ref{l:p-hall} implies that graph 
$B$ contains a perfect matching $M_1$. Label the edges of this matching with 
the number $1$, and take it out of $B$. We obtain an $(n-1)$-regular 
bipartite graph on which Lemma~\ref{l:p-hall} may be applied again. We keep 
creating matchings $M_2$, $M_3$, $\ldots$, in this fashion and label their 
edges with $2$, $3$, $\ldots$, until we arrive at $M_n$, when we stop. 
Notice that now each type $j \in T$ is connected to edges labeled with $1$ 
through $n$, and that each column $C_i$ is connected to all $n$ types of 
edges as well (in both cases exactly one from each type). For every $1 \le 
i \le n$ we rearrange the items in column $C_i$ such that the item 
corresponding to an edge labeled with $i$ goes into the $i^{\rm th}$ row. 
There will be no collisions by construction and we have arrived at the 
desired arrangement. 
\end{proof}

From the algorithmic perspective, each matching step in 
Lemma~\ref{l:column-permute} can be computed in expected $n\log n$ time 
\cite{goel2013perfect}. Alternatively, if a deterministic algorithm is 
desirable, a matching can be computed in $O(n^2)$ time \cite{cole2001edge}. 
The $n$ matchings can then be completed in $O(n^2\log n)$ expected time 
or $O(n^3)$ deterministic time. This is the dominant part of the 
algorithm's time complexity. 

\begin{figure}[h!]
	\begin{center}
  \begin{overpic}[width={\ifarxiv 0.7\columnwidth \else \columnwidth \fi}]{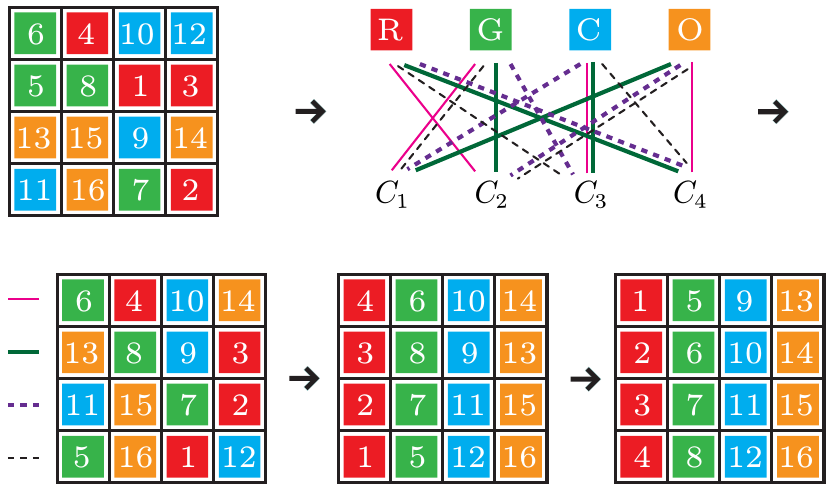}
	\put(11.5,28.5){{\small (a)}}
	\put(63.5,28.5){{\small (b)}}
	\put(17.5,-3.5){{\small (c)}}
	\put(50.8,-3.5){{\small (d)}}
	\put(84.5,-3.5){{\small (e)}}
	\end{overpic}
	\end{center}
  \caption{\label{figure:rubik} Illustration of applying the $n + n + n$ 
	shuffles as described in the proof of Theorem~\ref{t:rubik-table}. 
	(a) The initial $4\times 4$ table with a random arrangement of 16 items
	that are also colored (red, green, cyan, orange). (b) The bipartite graph 
	constructed from the table and a possible set of $4$ perfect 	matchings, 
	where $C_i, 1 \le i \le 4$ are the columns. As an example, green appears 
	twice in the first column of the table in (a) so there are two edges 
	between green and $C_1$. Each matching is marked with a unique line type
	(thin, thick, thick dashed, thin dashed). (c) Permuting each column according 
	to the matching results in each row containing each item type exactly once. 
	(d) Permuting each row of (c) then sorts all columns. (e) A final set of 
	$n$ column shuffles fully sorts the items. }
\end{figure}

\begin{figure*}[ht!]
	\begin{center}
  \begin{overpic}[width={\ifarxiv \columnwidth \else 2\columnwidth 
	\fi}]{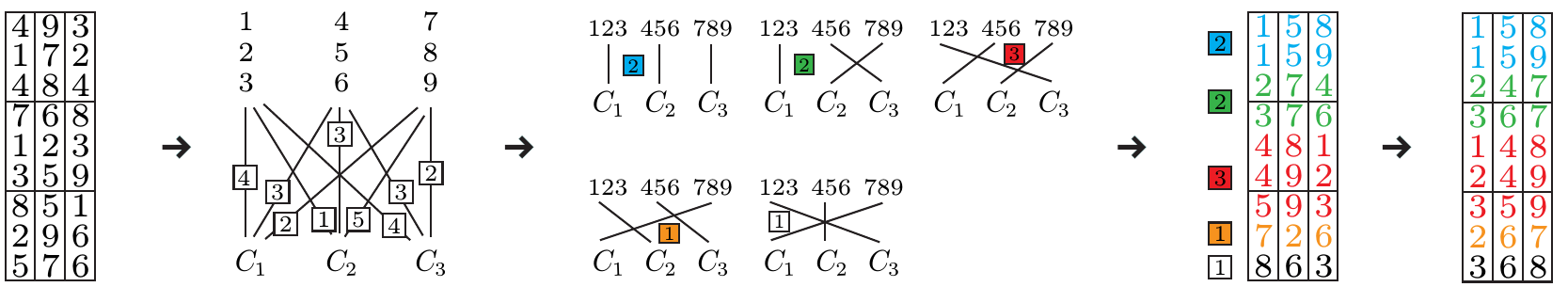}
	\put(2.2,-1.8){{\small (a)}}
	\put(20.8,-1.8){{\small (b)}}
	\put(51.8,-1.8){{\small (c)}}
	\put(81.4,-1.8){{\small (d)}}
	\put(95,-1.8){{\small (e)}}
	\end{overpic}
	\end{center}
  \caption{\label{figure:fat-rubik} Illustration of applying the first two 
	shuffle phases from the proof for Theorem~\ref{t:fat-rt}. (a) An initial 
	$3\times 3 \times 3$ fat Rubik Table with a random arrangement of $9$ types 
	of items. (b) The (weighted) regular 	bipartite graph from the setup in 
	(a). The numbers on edges denote the 	weight/multiplicity of the edges. 
	(c) The $5$ sets of (weighted) perfect 	matching extracted from (b). (d) 
	The fat-column permutations based on the matchings.	Notice that the numbers 
	in the columns remain the same between (a) and (d).	(e) The following 
	fat-row permutations which correctly sort the columns. With 	one more 
	round of fat-column permutations, we can sort the table so that each 
	cell contains a single item type. Additional distinguishability within 
	a cell is 	allowed as well.}
\end{figure*}

To provide some intuition of the shuffle algorithm, Fig.~\ref{figure:rubik} 
illustrates an application of the procedure used in proving 
Theorem~\ref{t:rubik-table} on a $4 \times 4$ Rubik Table containing $16$ 
items that also fall into four colors (types). 
Fig.~\ref{figure:rubik}(b) shows the constructed $4$-regular bipartite graph 
based on the configuration from Fig.~\ref{figure:rubik}(a). Four matchings 
are shown in different line patterns (thin solid, thick solid, thick dashed,
and thin dashed) and colors (pink, green, purple, and black). 
Based on the matchings, an intermediate Rubik Table is constructed as given 
in Fig.~\ref{figure:rubik}(c), where each column is shuffled as compared with 
Fig.~\ref{figure:rubik}(a). 
Note that each row now contains each type exactly once. 
After $n$ row shuffles, sorting by type is achieved as shown in 
Fig.~\ref{figure:rubik}(d). 
With one more round of $n$ column shuffles, Fig.~\ref{figure:rubik}(d) can be 
sorted to reach the fully sorted column-major configuration in 
Fig.~\ref{figure:rubik}(e).

It can be readily verified that Theorem~\ref{t:rubik-table} can be 
generalized to tables that are not squares. 

\begin{corollary}[(Linear shuffle algorithm for the Rubik Rectangle problem)]
Let $M$ be an $n \times m$ table filled with $nm$ unique items, $n \le m$. 
In $n$ row shuffles and $2m$ column shuffles, items in $M$ can be sorted 
arbitrarily. 
\end{corollary}

The matching-based shuffle routine is quite flexible, allowing many 
extensions to the Rubik Table problem. A particular useful extension is one 
for the fat Rubik Table problem as stated in Problem~\ref{prob:frt} and 
its labeled version, which adds a \emph{depth} to each table cell. With 
some relatively minor modifications, $2n$ or $3n$ shuffles are again sufficient.

\begin{theorem}[(Linear shuffle algorithm for fat Rubik Table problems)]
\label{t:fat-rt}
The fat Rubik Table problem and the labeled fat Rubik Table problem may be 
solved using $2n$ shuffles and $3n$ shuffles, respectively. 
\end{theorem}
\begin{proof}
The proof of Theorem~\ref{t:rubik-table} can be adapted with relatively 
minor changes. A similar three-phase procedure will be followed; again, the 
crucial part is the proof of the preparation phase, in which we show that 
we can permute the items within each fat column to reach the situation where 
the $nK$ items destined to go to any fixed fat column will end up in $nK$ 
positions, that are different when we project them to the first and third 
coordinates. The needed procedure for doing this is provided in 
Lemma~\ref{l:fat-column-permute}. 
\end{proof}

\begin{lemma}\label{l:fat-column-permute}
Let M be an $n\times n \times K$ table (row $\times$ column $\times$ depth) 
filled with items of $n$ different types. The number of items of type $j$ 
is exactly $nK$ for $1 \le j \le n$. Then we can permute the items within 
each fat column ($*, i, *$) of $M$ ($1 \le i \le n$) such that for any 
fixed type $j$ ($1 \le j \le n$), if we look at the $nK$ items of type 
$j$, they occupy distinct (row, depth) values when we project the triplet 
representing their new positions to the pair of row and depth coordinates.
\end{lemma}
\begin{proof}
The proof of the lemma is again based on applying Lemma~\ref{l:p-hall} on an 
$n + n$ bipartite graph. The nodes on the left are $n$ different types 
and the nodes on the right represent the fat columns. The edges correspond 
to the items, and we have $K$ parallel edges between right node $i$ and left
node $j$ as long as $K$ items need to go from fat column $i$ to fat column 
$j$. The only difference is that now the graph is $nK$-regular rather than 
$n$-regular. Again, we can iteratively decompose the edge-set of this 
bipartite graph into $nK$ perfect matchings, which gives the solution 
we are looking for.
\end{proof}

In terms of running time, the fat Rubik Table problems introduces an 
additional multiplicative factor $K$, yielding $O(n^2K\log n)$ expected or 
$O(n^3K)$ deterministic time. 

Again, to provide some intuition of the somewhat non-intuitive procedure, 
Fig.~\ref{figure:fat-rubik} illustrates an application of 
Theorem~\ref{t:fat-rt} to derive the first two sets of permutations for 
restoring order to a $3\times 3\times 3$ fat Rubik Table. In applying 
Lemma~\ref{l:fat-column-permute}, type $j$ corresponds to items numbered 
$(j-1)*3 + 1$ to $j*3$. For example, all items numbered $1-3$ are treated 
as type $1$. 

As for the Rubik Table problem, we mention that non-square fat Rubik Tables 
can be supported and leave it to the readers to fill in the details. 

Lastly, we examine a high-dimensional version of the Rubik Table problem. 
A fat version adding additional depth is again possible, which we do not 
further detail in this exposition. 

\begin{theorem}[(Rubik $R$-D Table)]\label{t:rrd}Let $M$ be an 
$\underbrace{n\times \ldots \times n}_R$ table, $R \ge 2$, filled with 
$n^R$ unique items. Assuming that any $(R-1)$-dimensional column can be 
arbitrarily shuffled, then $M$ can be arbitrarily sorted in $(2^R-1)n^{R-1}$ 
shuffles. 
\end{theorem}
\begin{proof}
Let $F(n,R)$ be the number of shuffles for given $n$ and $R$. We prove 
the claimed bound on $F(n,R)$ by induction on $R$. We can do the $2$ 
dimensional case in $3n$ shuffles by Theorem~\ref{t:rubik-table}. For $R 
> 2$, select out the first two dimensions and treat the remaining $R-2$ 
dimensions as the depth of a fat Rubik Table. By the induction hypothesis 
we can permute any fat column of $M$ any way we want in $F(n,R-1)$ 
shuffles (by the induction hypothesis, $F(n,R-1)=(2^{R-1}-1)n^{R-2}$). 
In the preparation phase we must do $n$ of these. Then we do $n^{R-1}$ 
row operations and finally we do again permutations on the fat columns, 
which cost $nF(n,R - 1)$. Altogether, we have 
\begin{align*}
F(n, R) &= 2nF(n,R - 1) + n^{R-1} \\
&= 2n(2^{R-1} - 1)n^{R-2} + n^{R-1} \\ 
&= (2^R-1)n^{R-1}.
\end{align*}
\end{proof}

\section{Application to Stack Rearrangement}\label{sec:stack}
In \cite{HanStiBekYu2017}, an $O(nd(\log n + \log d))$ upper bound 
is established for \lsr. Their method is sequential: it first sorts items
into the correct stacks and then sort each stack individually. 
In Sec.~\ref{subsec:logd}, we start with describing a more efficient 
approach that improves the bound for $n \ge d$ to $O(nd \log d)$ with the 
constant in the big $O(\cdot)$ notation given explicitly. This is achieved interleaving 
inter- and intra-stack sorting. This result also uses ideas that enable 
the Rubik Table results but in a more direct manner. As such, the proof 
techniques may be of independent interest. 
Then, in Sec.~\ref{subsec:nd1} and Sec.~\ref{subsec:nd2}, we show how the 
Rubik Table results allow us to further lower the upper bound to $O(nd)$ 
for $d \le n^{\frac{m}{2}}$ with $m > 0$ being an arbitrary constant. 
We briefly discuss the cases of having constant $n$ or $d$ in 
Sec.~\ref{subsec:cnd}.

\subsection{A $7nd(1 + \log d)$ Upper Bound for $n \ge d$}\label{subsec:logd}
As stated, this subsection seeks to establish the following.
\begin{proposition}[(A $7nd(1 + \log d)$ upper bound for \usr with $n \ge d$)]
\label{p:7l}
For arbitrary $n \ge d$ where $d = 2^{\lceil \log d \rceil}$, a \usr instance 
can be solved using $7nd(1 + \log d)$ pop-pushes. 
\end{proposition}

For simplicity, assume that $d$ is a two-power, i.e., $d = 2^{\lceil \log
d \rceil}$. If $d$ is not a two-power, we lose at most a factor of $2$ in 
terms of efficiency. The algorithm is recursive, and we recurse on $\log 
d$. When $d = 1$, there is a simple solution that uses less than 
$\frac{3}{2}n$ pop-pushes. Let $F(n, d)$ be the number of pop-push 
operations used by our algorithm, we will prove the following recurrence:
$$
F(n,d) = 2F(n, \frac{d}{2}) + 7nd. 
$$
This will clearly give the claimed efficiency because $7nd(1 + \log d)
= 7nd + 2(7n\frac{d}{2})(1 + \log \frac{d}{2})$.
We denote the content of a stack $S_i$ with $S_i[1], \ldots, S_i[d]$, 
where $S_i[d]$ refers to the top of stack $S_i$. The buffer stack is 
denoted as $S_{\emptyset}$.
The recurrence will be proved using two intermediate lemmas. We will 
need the notion of a balanced arrangement. 

\begin{definition}[(Balanced arrangement)] An arrangement, where each stack 
$S_i, 1 \le i \le n$, holds d items, is \emph{balanced} if for all types 
$j \in \{1,\ldots,n\}$, the number of items of type $j$ elements held in 
$S_1[1..\frac{d}{2}] \cup \ldots \cup S_n[1..\frac{d}{2}]$ is exactly 
$\frac{d}{2}$.
\end{definition}

\begin{lemma}
Assume we have an arbitrary arrangement of the items in the stacks $S_1,
\ldots, S_n$, where each $S_i$ holds exactly $d$ items. Then, we can permute 
the contents of each stack individually such that the arrangement becomes 
balanced.
\end{lemma}
\begin{proof}
Create a bipartite graph $B(T,S)$ on $n + n$ nodes such that the left side,
$T$, stands for all the types $\{1,\ldots,n\}$, and the right side, $S$, 
stands for all the stacks $S_1, \ldots, S_n$. We draw $k$ edges between 
type $j$ and stack $i$ if stack $i$ contains $k$ elements of type $j$. 
Notice that $B$ is $d$-regular from both sides. Hall’s theorem implies that 
graph $B$ contains a perfect matchings $M_1$. We remove $M_1$ and repeat 
the process to get $\frac{d}{2}$ matchings $M_1, \ldots, M_{\frac{d}{2}}$. 
Color these matchings blue. 
Notice that now in $B$, each type $T_j$ is connected to exactly 
$\frac{d}{2}$ blue edges, and that each stack $S_i$ is connected to 
$\frac{d}{2}$ blue edges as well. For every $1 \le i \le n$ rearrange the 
items in stack $S_i$ such that the items corresponding to the blue edges 
occupy $S_i [1..\frac{d}{2}]$. We have arrived at a balanced arrangement.
\end{proof}

Next, we show how a balanced arrangement can be realized efficiently.

\begin{lemma}\label{l:swap}
For $n \ge d$, starting from an arbitrary arrangement, we can achieve 
a balanced arrangement using at most $4nd$ pop-pushes.
\end{lemma}
\begin{proof}
It is sufficient to show that any permutation of stack $i$ can be implemented 
in $4d$ pop-pushes:
\begin{itemize}
\item Select $d-1$ stacks other than $S_i$, and put their top elements 
(temporarily) into $S_{\emptyset}$. At the end of this step, 
$S_{\emptyset}$ holds $d - 1$ items and have one more space available.
\item Put the elements of $S_i$ into the thus freed tops of the stacks, 
including one on the top of $S_{\emptyset}$. 
\item Put the elements back to $S_i$ in the desired order.
\item Restore the tops of the stacks from $S_{\emptyset}$.
\end{itemize}
Each of the phases above takes $d$ pop-pushes to execute, altogether $4d$ 
pop-pushes.
\end{proof}

With Lemma~\ref{l:swap}, we are ready to outline the algorithm that backs
up Proposition~\ref{p:7l}, given in Alg.~\ref{alg:logd}.

\vspace{1mm}
\begin{algorithm}
\vspace{0.025in}

Rearrange the stacks to achieve a balanced partition using Lemma~\ref{l:swap}.

\vspace{0.025in}
Forget about the bottom $\frac{d}{2}$ elements of each stack and 
solve the problem on the top with $SR(n,\frac{d}{2})$.

\vspace{0.025in}
Swap the top $\frac{d}{2}$ elements with the bottom $\frac{d}{2}$ elements 
in each $S_i$.

\vspace{0.025in}
Forget about the bottom $\frac{d}{2}$ elements of each stack and 
solve the problem on the top with $SR(n,\frac{d}{2})$.

\vspace{0.025in}
\caption{SR($n, d$)} \label{alg:logd}
\end{algorithm}
\vspace{1mm}

It is clear that the algorithm returns the correct solution. The number of 
pop-pushes used, line by line, is:
\begin{align*}
F(n,d) & = \underbrace{\vphantom{\frac{a}{b}} 4nd}_{line\;1} + 
\underbrace{\vphantom{\frac{a}{b}}F(n,\frac{d}{2})}_{line\;2} + 
\underbrace{\vphantom{\frac{a}{b}}3nd}_{line\;3} + 
\underbrace{\vphantom{\frac{a}{b}}F(n,\frac{d}{2})}_{line\;4} \\
& = 2F(n, \frac{d}{2}) + 7nd.
\end{align*}

Because each stack can be sorted by label in $d\log d$ pop-pushes 
(see Lemma IV.3. in \cite{HanStiBekYu2017}), adding another $nd\log d$ 
pop-pushes can solve \lsr. 

\begin{corollary}[(A $7nd + 8nd\log d$ upper bound for \lsr with $n \ge d$)]
\label{c:7l}
For arbitrary $n \ge d$, an \lsr instance can be solved using 
$7nd + 8nd\log d$ pop-pushes. 
\end{corollary}

\subsection{Linear pop-push algorithm for \lsr,  $d \le \sqrt{n}$}
\label{subsec:nd1}

Results on fat Rubik Table problem lead to significantly improved upper 
bounds for \usr and \lsr that largely match the lower bound (asymptotically), 
which we establish in this section. The proposed algorithmic approach 
applies directly to \lsr and therefore \usr. The improved upper bounds are 
obtained through recursive applications of the fat Rubik Table result 
(Theorem~\ref{t:fat-rt}) through ``simulated'' fat Rubik Table 
column and row permutations. The recursion is done based on increasing 
$2\frac{\log d}{\log n}$. We first address the case of $2\frac{\log d}
{\log n} \le 1$ (i.e., $d \le \sqrt{n}$), followed by the case 
$2\frac{\log d}{\log n} \le 2$ (i.e., $d \le n$), and finally the general 
case of $2\frac{\log d}{\log n} \le m$ (i.e., $d \le n^{\frac{m}{2}}$).


\begin{lemma}[(Linear pop-push algorithm for \lsr, $d = \sqrt{n}$)]\label{l:d-sqrtn}
\lsr with $d = \sqrt{n}$ can be solved using $O(nd)$ pop-pushes. 
\end{lemma}
\begin{proof}
We construct an $n' \times n' \times K$ fat Rubik Table with $n' = K = d = 
\sqrt{n}$. A depth $K=d$ fat cell of the table with index $(i, j), 1 \le i,j 
\le n' = d$ is identified with the stack indexed $(j-1)*d + i$ (see
Fig.~\ref{figure:simu-rubik} for an example), which ranges between $1$ 
and $n = d^2$. 

\begin{figure}[ht!]
	\begin{center}
  \begin{overpic}[width={\ifarxiv 0.6\columnwidth \else 1\columnwidth \fi}]
	{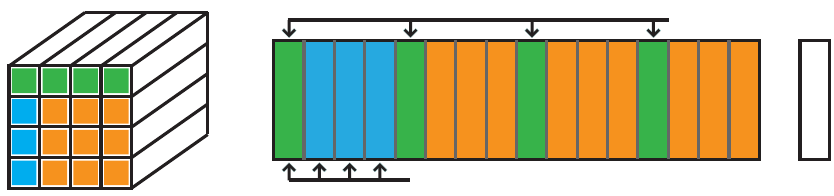}
  \put(81.4,20.3){{\small a fat row}}
  \put(51,0.8){{\small a fat column}}
	\end{overpic}
	\end{center}
  \caption{\label{figure:simu-rubik}Correspondence between a $d \times d 
	\times d$ fat Rubik Table and the $n = d^2$ stacks of depth $d$ in a stack 
	rearrangement problem instance. Here, $d = 4$.}
\end{figure}

We first show that we can simulate a single fat column permutation of $n'K 
= \sqrt{n}d = d^2$ items in  $O(d^2)$ pop-pushes, which can be 
achieved by:
\begin{enumerate}
\item Moving the content of $\sqrt{n} = d$ stacks to the top of the $n$ 
stacks using $O(d^2)$ pop-pushes. For each stack, we may move its 
content to the top of other stacks using the operations illustrated in the 
first four figures in Fig.~\ref{figure:d-sqrt-n}, which takes $3d$ 
pop-pushes. Applying this to $d$ stacks requires $3d^2$ pop-pushes, 
resulting the configuration shown in the fifth figure (bottom left) of 
Fig.~\ref{figure:d-sqrt-n}.
\item Sort the $d^2$ elements on top of the stacks arbitrarily, which 
takes $O(d^2)$ pop-pushes. This requires using the buffer stack to hold at 
most one item temporarily. This happens in the fifth 
figure of Fig.~\ref{figure:d-sqrt-n}.
\item Revert the first step above to return the sorted $d^2$ items to 
the $d$ stacks of current interest. This corresponds to going from the 
fifth figure to the last figure in Fig.~\ref{figure:d-sqrt-n}. 
\end{enumerate}

\begin{figure}[h]
	\begin{center}
  \begin{overpic}[width={\ifarxiv 0.7\columnwidth \else \columnwidth \fi}]
	{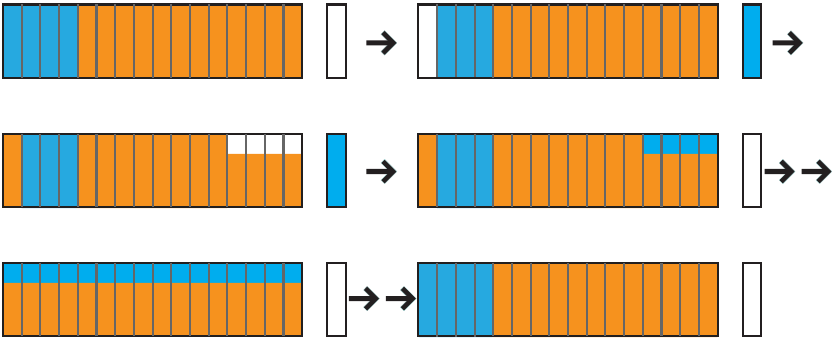}
	\end{overpic}
	\end{center}
  \caption{\label{figure:d-sqrt-n}Illustration of the steps for realizing 
	a simulated fat column permutation in	$O(n'K) = O(d^2)$ pop-pushes. The cyan 
	stacks (left four stacks in the first figure) are the $d$ stacks of 
	interest. 
	The first step (indicated by the arrow) illustrates emptying the leftmost 
	stack to the buffer. 
	Then, the top of some stacks not of current interest (the orange ones) 
	can be moved to the emptied stack (second step). 
	Subsequently, the buffer content can be put on the top of stacks (third
	step). After this is done for all stacks of current interest, the 
	contents of these stacks are moved to the top of the $d^2$ stacks (fourth
	step, marked with double arrows ``$\to \to$''). 
	After rearranging these items as needed, they can then be returned (fifth 
	step, marked with double arrows ``$\to \to$''). 
	The ``simulated'' fat-column shuffle mirrors the step of permuting the 
	first column of Fig.~\ref{figure:fat-rubik}(a) to the first column of 
	Fig.~\ref{figure:fat-rubik}(d). 
	}
\end{figure}

Following the same procedure, a fat row permutation can also be carried out 
in $O(d^2)$ pop-pushes. To apply Theorem~\ref{t:fat-rt}, we partition
all $nd = d^3$ items into $d$ types where items of type $t$, $1 \le t \le d$,
have destinations in stack $(t-1)d + 1$ to stack $td$. By 
Theorem~\ref{t:fat-rt}, using $d$ fat column permutations and 
$d$ fat row permutations, all items of type $t, 1\le t \le d$ can be moved 
to fat column $t$. Then, applying a fat column permutation to a fat column 
$t$ can sort items in the fat column arbitrarily. This solves the \lsr 
problem (and therefore, a \usr problem). 

Tallying the number of pop-pushes, we have done $3d$ fat column/row permutations, 
each of which takes $O(d^2)$ stack pop-pushes. The total is then 
$O(d^3) = O(nd)$ (with more careful counting, we can conclude that the 
number of pop-pushes is bounded by $27nd$).
\end{proof}

It is straightforward to see that Lemma~\ref{l:d-sqrtn} readily 
generalizes to $d < \sqrt{n}$. If $n$ is a square, then the corollary 
directly applies. For $n$ that is not a square, e.g., $n = q^2 + p$ where
$q^2$ is the largest square less than $n$, we can partition the 
$n$ stacks into two groups of $q^2$ stacks each with $q^2 - p$ of the 
stacks overlapping between the two groups (we can assume that $n$ is 
sufficiently large so that $q^2 - p > p$; otherwise $n$ can be treated as 
a constant). Focusing on the first group of $q^2$ stacks, we can then 
apply Lemma~\ref{l:d-sqrtn} (note that $m$ satisfies $\sqrt{n} >
q > \lceil \sqrt{n} \rceil - 1 \ge d$) to ``concentrate'' items that should 
go to the rest $p$ stacks in the $q^2 - p$ stacks shared between the two 
groups. Then, Lemma~\ref{l:d-sqrtn} can be applied again to the second 
group of $q^2$ stacks in a similar fashion, followed by one last application 
to the first group of $q^2$ stacks, which solves the entire problem. We have
proved

\begin{corollary}[(Linear pop-push algorithm for \lsr, $d \le \sqrt{n}$)]\label{c:d-le-sqrtn}
\lsr with $d \le \sqrt{n}$ can be solved using $O(nd)$ pop-pushes. 
\end{corollary}

Another consequence of Theorem~\ref{t:fat-rt} is that, if we 
allow $b = \lceil \sqrt{n} \rceil$ empty buffer stacks (instead of a single buffer stack)
of depth $d$ each, \usr with arbitrary $n$ and $d$ can be solved using $O(nd)$ 
pop-pushes. This is true because a constrained (items are distinguishable by types
but do not have individual labels) fat column permutation can be readily 
executed in $2\sqrt{n}d$ pop-pushes using $\lceil \sqrt{n} \rceil$ buffer stacks. 

\begin{corollary}[(Linear pop-push algorithm for \usr with extra buffers)]
\label{c:d-le-sqrtn} Given $b = \lceil \sqrt{n} \rceil $ buffer stacks, 
\usr with arbitrary but sufficiently large $n$ and $d$ can be solved using 
$O(nd)$ pop-pushes. 
\end{corollary}

If $n$ is a perfect square, then the number of required pop-pushes is bounded 
by $6nd$. It is not clear that having $\lceil \sqrt{n} \rceil$ buffers help with solving 
\lsr in $O(nd)$ time for arbitrary $n$ and $d$; we leave this as an open 
question.  

\subsection{Linear pop-push algorithm for \lsr, $d < n^{\frac{m}{2}}$ and Constant $m > 0$}
\label{subsec:nd2}
We continue to look at the case where $2\frac{\log d}{\log n} > 1$, starting
with $n = d = k^2$ for some integer $k$. The algorithm for doing 
so will invoke Lemma~\ref{l:d-sqrtn} repeatedly, which uses the top $k$ rows 
of the stacks. 

\begin{lemma}[(Linear pop-push algorithm for \lsr, $d = n$)]\label{l:n-eq-d}
For $n = d = k^2$, \lsr can be solved in $O(nd)$ pop-pushes.  
\end{lemma}
\begin{proof}
Similar to how Lemma~\ref{l:d-sqrtn} is proven, we will simulate column and 
row permutations on a fat Rubik Table mapped to the stack rearrangement 
instance. To do the mapping, we simply identify stacks $(i-1)k + 1, \ldots
ik$ with the $i^{\mathrm{th}}$ fat column of the fat Rubik Table. The 
$j, j + k, j + 2k, \ldots, j+(d-1)k$ stacks are identified with the 
$j^{\mathrm{th}}$ fat row. It is clear that, if we can simulate fat 
column/row permutations using $O(k^3)$ pop-pushes, then the statement of the lemma 
holds. 

To simulate a fat column/row permutation, we note that the content of any $k$ 
stacks can be flipped with the contents of the top $k$ rows of the $k^2$ stacks, 
using the buffer stack. This takes $O(k^3)$ pop-pushes and is illustrated 
in Fig.~\ref{figure:n-eq-d}(a) $\to$ Fig.~\ref{figure:n-eq-d}(e), which is 
similar to the procedure illustrated in Fig.~\ref{figure:d-sqrt-n} (if we 
``compress'' $k$ consecutive items in a stack into a single item). Once the 
contents of the selected $k$ stacks (corresponding to a fat column/row) occupy 
the top $k$ rows of the 
$k^2$ stacks, Lemma~\ref{l:d-sqrtn} may be applied to rearrange the items 
in them arbitrarily, which takes $O(k^3)$ time as well. A reversal of the 
first step then completes a simulated fat column/row permutation. The total 
number of pop-push operations used is $O(k^3)$. 
\begin{figure}[ht!]
	\begin{center}
  \begin{overpic}[width={\ifarxiv 0.7\columnwidth \else \columnwidth \fi}]
	{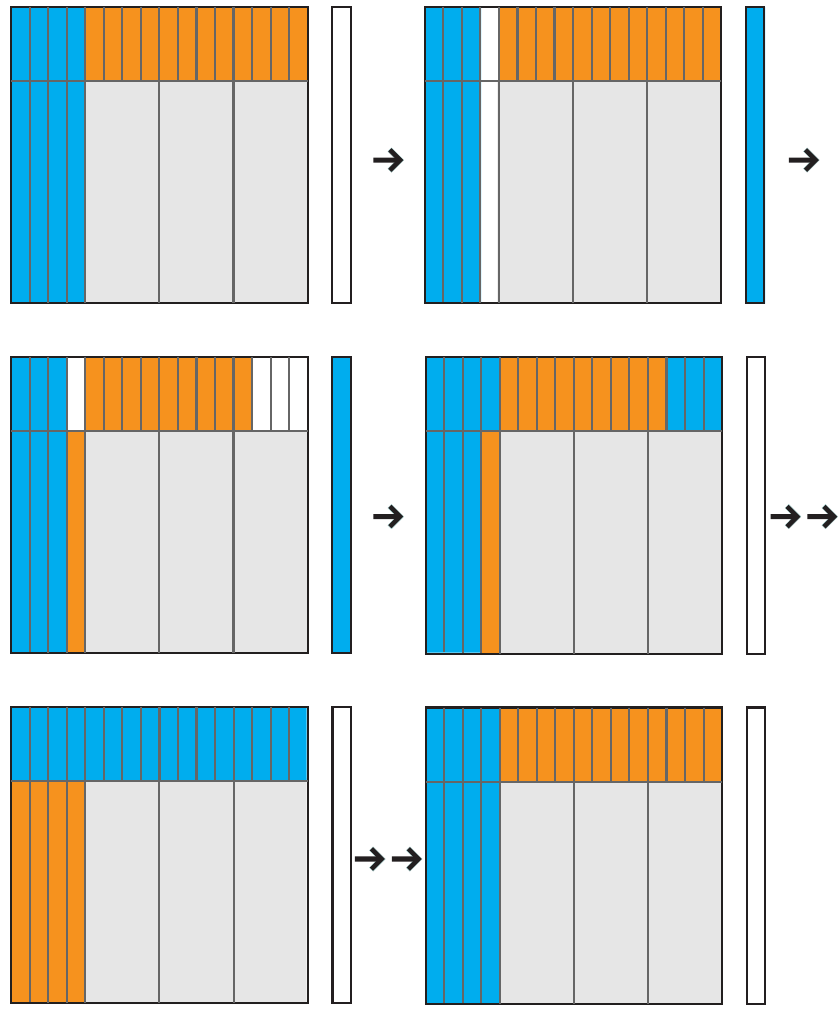}
	\put(17,-3){{\small (e)}}
	\put(58.5,-3){{\small (f)}}
	\put(17,32){{\small (c)}}
	\put(58.5,32){{\small (d)}}
	\put(17,66.5){{\small (a)}}
	\put(58.5,66.5){{\small (b)}}
	\end{overpic}
	\end{center}
	\vspace*{2mm}
  \caption{\label{figure:n-eq-d} Illustration of a simulated fat column 
	permutation over $n$ stacks of depth $d$, with $n = d = k^2$. (a) The 
	cyan colored stacks (left four in the first figure) map to a fat column 
	of a fat Rubik Table. (b)-(e) 
	In $O(k^3)$ stack pop-pushes, a fat column (or a fat row) can be moved 
	to the top $k$ rows of the $n$ stacks, over which Lemma~\ref{l:d-sqrtn} 
	can be applied to perform arbitrary sorting in $O(k^3)$ pop-pushes. 
	(f) After rearrangement, the fat column (or fat row) can be restored 
	using the same procedure in $O(k^3)$ pop-pushes, completing the shuffle
	operation. 
	}
\end{figure}
\end{proof}

It is clear that Lemma~\ref{l:n-eq-d} continues to apply when 
$\sqrt{n} < d < n$, following the same argument used for establishing 
Corollary~\ref{c:d-le-sqrtn}. That is, 

\begin{proposition}[(Linear pop-push algorithm for \lsr, $d \le n$)]\label{p:lsr-d-le-n}
\lsr with $d \le n$ can be solved using $O(nd)$ pop-pushes. 
\end{proposition}

The condition $d = n$ in Lemma~\ref{l:n-eq-d} may be viewed as 
$\frac{\log d}{\log n} = 1$ or $d = n^{\frac{m}{2}}$ with $m = 2$. 
Taking a closer look at the proof for Lemma~\ref{l:n-eq-d}, it is 
straightforward to see that the same argument directly extends to 
show that the \lsr case of $d=k^3$ and $n = k^2$ 
($\frac{\log d}{\log n} = \frac{3}{2}$) can be solved using 
$O(nd)$ pop-pushes for any positive integer $k$. In proving Lemma~\ref{l:n-eq-d}, 
the top $k$ rows of the stacks are used as a \emph{swap space} for 
applying Lemma~\ref{l:d-sqrtn}, simulating a fat column/row permutation. 
In a similar fashion, for $d = k^3$ and $n = k^2$, the top $k^2$ rows 
can be used as the swap space, which allows us to work with a total of 
$k^2\cdot k^2 = k^4$ items. Once the swap space is properly set up, the 
$k^4$ items can be rearranged arbitrarily by Lemma~\ref{l:n-eq-d} using 
$O(k^4)$ pop-pushes. So \lsr with $d = n^{\frac{m}{2}}$ for $m = 3$
can be solved in $O(nd)$ pop-pushes. Corollary~\ref{c:d-le-sqrtn} then 
generalizes to apply to all cases where $\frac{\log d}{\log n} \le 
\frac{3}{2}$. \vspace{1mm}

Recursively, Lemma~\ref{l:n-eq-d} may be generalized to arbitrary 
$m \ge 2$. For $m = 3$, the procedure will call the $m = 2$ case $3k$ 
times. If the $n = d$ case requires $cnd = ck^4$ pop-pushes for some constant 
$c$, then the $m = 3$ case will need $3ck^5$ pop-pushes. Recursively, for 
general $m$, the recursive procedure will require about $3^m cnd$ pop-pushes 
for $d = n^{\frac{m}{2}}$. We have proved

\begin{theorem}[(Algorithm for \lsr with $d \le n^{\frac{m}{2}}$ and 
$m \ge 0$)]\label{t:lsr-d-ge-n} \lsr with $d \le n^{\frac{m}{2}}$ for 
$m \ge 0$ can be solved using $O(3^mnd)$ pop-pushes. 
\end{theorem}

For any fixed $m \ge 2$, it is clear that \lsr can be solved in $O(nd)$ 
pop-pushes for $n^{\frac{m-1}{2}} < d < n^{\frac{m}{2}}$, possibly with a 
larger constant than the $d = n^{\frac{m}{2}}$ case. For fixed $m$, 
$3^m$ is also a constant. Summarizing the results on the upper bounds 
obtained so far, we have

\begin{theorem}[(Linear pop-push algorithm for \lsr, $d \le  \lceil cn \rceil$)]
\label{t:lsr-d-cn}For arbitrary fixed real number $c > 0$, \lsr with 
$d \le \lceil cn \rceil$ can be solved using $O(nd)$ pop-pushes. 
\end{theorem}


For \usr with $d = n = k^2$, with additional care in carrying out 
the recursive procedure, we only need to make $2k$ calls to 
Lemma~\ref{l:d-sqrtn} instead of $3k$ as required in proving 
Lemma~\ref{l:n-eq-d}. This gives us that \usr with $d \le n^{\frac{m}{2}}$ 
for $m \ge 2$ can be solved using $O(2^mnd)$ pop-pushes instead of 
the $O(3^mnd)$ stated in Theorem~\ref{t:lsr-d-ge-n}. We omit the 
very involved procedure, which boils down to doing a mixed column 
and row permutation. 



\subsection{Constant $n$ or $d$}\label{subsec:cnd}
Lastly, we briefly discuss what happens when $n$ or $d$ is a constant.
An $O(nd\log n)$ algorithm for \usr is provided in \cite{HanStiBekYu2017} 
for arbitrary $n$ and $d$, using divide and conquer over the number of 
stacks $n$. This implies that for constant $n$, $O(d)$ pop-pushes is sufficient, 
matching the $\Omega(nd)$ lower bound. For constant $d$, each stack can 
be sorted in $O(1)$ pop-pushes by first moving all type $k$ items to the top of 
the stacks they are at (for a stack $i$ that contains type $k$ item, this 
can be done by first moving the top item from some $d$ stacks to the 
buffer, moving items in stack $i$ to the empty $d$ top spots, and then 
moving them back to stack $i$ so that type $k$ items stay on the top). 
Then type $k$ items can be all moved to the buffer stack and followed 
by emptying stack $k$, then to stack $k$. This yields an $O(n)$-pop-push
algorithm, also matching the lower bound.

\section{Application to Optimal Multi-Robot Motion Planning}\label{sec:mrmp}
The built-in global coordination perspective of Rubik Table problems 
naturally extends them to applications toward optimal multi-robot 
motion planning (\mrmp) as well. In \cite{Yu18RSS}, an $O(1)$ 
makespan-optimal
algorithm, \emph{partition and flow} (\paf), is proposed for solving 
makespan-optimal \mrmp for the grid setting illustrated in 
Fig.~\ref{figure:mrmp-ex}. 
\paf in turn utilizes an expected $O(1)$-optimal algorithm, 
\emph{split and group} (\sag), that recursively splits the problem into 
smaller pieces. 
\sag uses (non-labeled) Rubik Table results. In Sec.~\ref{subsec:rss}, 
we summarize how Rubik Table results are applied to enable \sag as is 
done in \cite{Yu18RSS}. 
Then, in Sec.~\ref{subsec:rb-mrmp}, we describe a new and more direct 
expected $O(1)$-optimal algorithm as an application of the labeled Rubik 
Table result. 
We mention that, whereas our illustration will focus on the 2D setting, 
the methods generalize to higher dimensions. Because the generalization 
is fairly straightforward, we do not further elaborate here. 

\begin{figure}[h]
	\begin{center}
  \begin{overpic}[width={\ifarxiv 0.7\columnwidth \else \columnwidth \fi}]
	{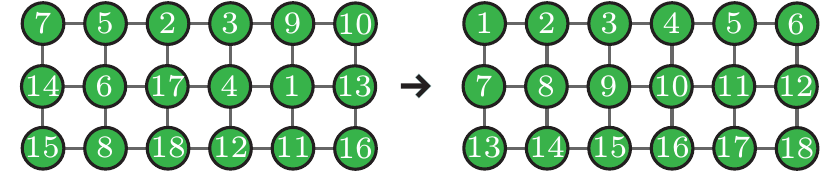}
	\end{overpic}
	\end{center}
  \caption{\label{figure:mrmp-ex} An example \mrmp instance on a $3 \times 
	6$ grid where $18$ robots in an arbitrary configuration on the left must 
	be routed to take the row-major ordering on the right, which is equivalent
	to moving between two arbitrary configurations.}
\end{figure}

\subsection{Recursive Split and Group}\label{subsec:rss}
For a random \mrmp instance (e.g., Fig.~\ref{figure:mrmp-ex}), the expected 
makespan is equal to the sum of the grid's dimensions. For 2D, this is $O(m_1
+ m_2)$ (see Sec.~\ref{subsec:mrmp-d} for the definition of \mrmp on grids). 
To compute a routing solution with expected $O(1)$-optimal makespan, the 
total number of steps must be bounded by $O(m_1 + m_2)$. 
To achieve this, the \emph{split and group} or \sag algorithm from 
\cite{Yu18RSS} \emph{splits} the $m_1 \times m_2$ grid along a longer 
dimension into two equal (or roughly equal) sized grids, and then route 
the robots so that they move to the half that they belong to in the goal 
configuration.
This induces a \emph{grouping} problem where each robot must either stay in 
the current half grid it is in or move to the other half grid.  
For the instance given in Fig.~\ref{figure:mrmp-ex}, the grouping 
problem is given in Fig.~\ref{figure:mrmp-sag}. 

\begin{figure}[h]
	\begin{center}
  \begin{overpic}[width={\ifarxiv 0.8\columnwidth \else \columnwidth \fi}]
	{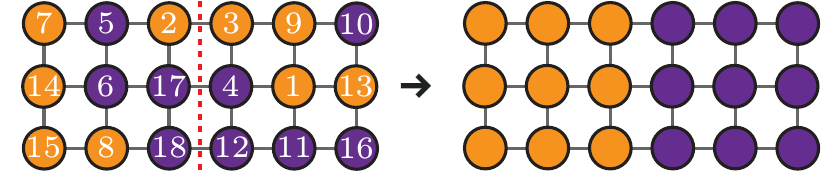}
	\end{overpic}
	\end{center}
  \caption{\label{figure:mrmp-sag} Grid splitting induces a grouping problem
	where robots must be routed to one of the two half grids. Here, the labels 
	do not matter in the goal configuration; only the group (color) matters.
	That is, the goal configuration for this induced problem is not unique.}
\end{figure}

We claim this grouping operation can be achieved in $O(m_1 + m_2)$ steps 
(or makespan). 
Then, assuming without loss of generality that $m_1 \le m_2$, after the 
first iteration, two half grids of dimensions about $\frac{m_1}{2} \times 
m_2$ are obtained. Invoking \sag on each of the two sub-problems, the second 
iteration takes $O(\frac{m_1}{2} 
+ m_2)$ steps because the two sub-problems can be solved in parallel. The 
overall recursion takes $O(m_1 + m_2) + O(\frac{m_1}{2} + m_2) + 
O(\frac{m_1}{2} + \frac{m_2}{2}) + \ldots = O(m_1 + m_2)$ steps, as desired. 

To carry out the $O(m_1 + m_2)$-step grouping operation in the first 
iteration as just claimed (e.g., Fig.~\ref{figure:mrmp-sag}), motion 
primitives are first built, starting from a $3$-step pairwise swaps on a 
$2 \times 3$ grid, as shown in Fig.~\ref{figure:23}. 

\begin{figure}[h]
	\begin{center}
  \begin{overpic}[width={\ifarxiv 0.8\columnwidth \else \columnwidth \fi}]
	{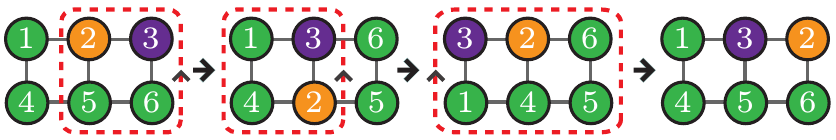}
	\end{overpic}
	\end{center}
  \caption{\label{figure:23} In three steps, robots $2$ and $3$ can be ``swapped''.}
\end{figure}

Then, multiple of these pairwise swaps can be combined, with parallel 
executions, to enable the swapping of two groups of robots of equal number 
on a path of length $\ell$ embedded in a grid (see, e.g., 
Fig.~\ref{figure:mrmp-line}). This can be completed in $O(\ell)$ steps. Note 
that these two groups to be exchanged may overlap; let us denote the group 
swapping procedure on an embedded path as \emph{line swap}. Multiple line 
swaps can be carried out in parallel on parallel paths. 

\begin{figure}[h]
	\begin{center}
  \begin{overpic}[width={\ifarxiv 0.8\columnwidth \else \columnwidth \fi}]
	{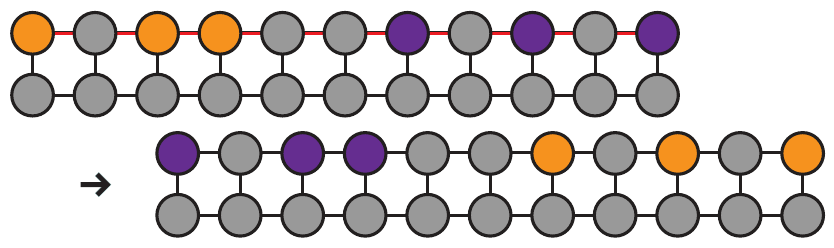}
	\end{overpic}
	\end{center}
  \caption{\label{figure:mrmp-line} The orange and purple robots on the 
	red path of length $\ell$ can be exchanged in $O(\ell)$ steps.}
\end{figure}

Using line swaps, to realize the grouping procedure (e.g., 
Fig.~\ref{figure:mrmp-sag}), we may simply convert it to a Rubik Table 
(more precisely, Rubik Rectangle) problem by relabeling the robots to form 
$m_2$ types. One possible relabeling is given in Fig.~\ref{figure:mrmp-sagr}. 
The Rubik Table results then readily apply to produce three sets of column 
and row shuffles that can be simulated using line swaps, with each set 
executed in parallel, guaranteeing the $O(m_1+m_2)$-step bound. 

\begin{figure}[h]
	\begin{center}
  \begin{overpic}[width={\ifarxiv 0.8\columnwidth \else \columnwidth \fi}]
	{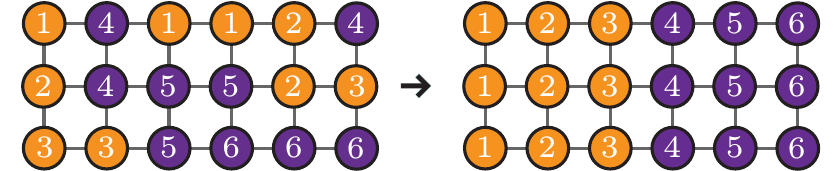}
	\end{overpic}
	\end{center}
  \caption{\label{figure:mrmp-sagr} Converting the grouping problem from 
	Fig.~\ref{figure:mrmp-sag} to a Rubik Rectangle problem with six 
	``types''. }
\end{figure}

\subsection{A New Expected $O(1)$-Optimal Algorithm}\label{subsec:rb-mrmp}
In the $i^{th}$ iteration of the \sag algorithm, $2^i$ calls of the Rubik 
Table results must be invoked. Here, we describe a simplified algorithm 
that only requires a single call to the labeled Rubik Table result. As a
preparation step, we first update the line swap motion primitive to apply 
to fully labeled robots instead of just two groups of robots. 

\begin{lemma}[(Labeled line swap algorithm)]\label{l:lls}
The arbitrary reconfiguration of labeled robots on a path of length $\ell$ 
embedded in an $m_1 \times m_2$ grid with $m_1, m_2 \ge 2$ and $m_1m_2 \ge 
6$ can be realized in $O(\ell)$ steps. 
\end{lemma}
\begin{proof}
The labeled line swap can be realized by recursive and parallel calls to 
the line swap procedure. For the first iteration, the path is split in 
the middle, which induces a partition of the robots; some must have goals 
in the current half they are on and the others must have goals on the 
opposite half. This leads to a problem that is exactly the same as that 
illustrated in Fig.~\ref{figure:mrmp-line}, which can be solved by the 
line swap algorithm in $O(\ell)$ steps, yielding two sub-problems on two 
paths of length about $\frac{\ell}{2}$ each, that can be solved in 
parallel recursively. All together, the required number of steps for 
solving the entire problem is then $O(\ell) + O(\frac{\ell}{2}) + 
O(\frac{\ell}{4}) + \ldots = O(\ell)$. 
\end{proof}

The labeled line swap routine allows the simulation of both 
column and row shuffles as required for solving a labeled Rubik Table 
problem. Given an \mrmp problem on an $n \times n$ square grid, 
Theorem~\ref{t:rubik-table} can readily compute the $n + n + n$ 
shuffles for solving the \mrmp problem as a Rubik Table problem. Each 
set of $n$ shuffles can be executed as labeled line swaps in parallel,
requiring a makespan of $O(n)$. Clearly, the same process works on an 
$m_1 \times m_2$ grid, yielding:

\begin{proposition}[(Expected $O(1)$-optimal algorithm for \mrmp via 
Rubik Table)]\label{p:mrmpn}
\mrmp on an $m_1 \times m_2$ grid with $m_1, m_2 \ge 2$ and $m_1m_2 \ge 
6$ can be solved in $O(m_1 + m_2)$ steps, using a single call to the 
labeled Rubik Table algorithm. 
\end{proposition}

\section{Conclusion and Discussion}\label{sec:conclusion}
In studying decision-making problems for many mobile objects including 
stack rearrangement and multi-robot motion planning, we propose an 
abstract problem involving the shuffling of columns and rows of an $n 
\times n$ table containing $n^2$ items, which we call the Rubik Table
problem. We show that, surprisingly, the 
Rubik Table problems can be solved using a minimal number of column and 
row shuffle operations. Then, through simulating these column and row 
shuffle operations, more efficient methods are developed for problems 
involving moving many objects. Specifically, for stack rearrangement 
problems, the fat Rubik Table result applies recursively to yield 
an $O(nd)$ pop-push algorithm for all $d \le n^{\frac{m}{2}}$ where 
$m > 0$ is 
an arbitrary constant, which meets the lower bound. For multi-robot 
motion planning, Rubik Table results facilitate the iterative split
and group of two groups of robots, leading to an \emph{expected}
$O(1)$ makespan-optimal algorithm, which can be used to drive an 
$O(1)$ makespan-optimal algorithm for the same. This work further 
provides an updated expected $O(1)$-optimal algorithm using a single
call to the labeled Rubik Table algorithm. 

We conclude the work by discussing several interesting open questions
for the readers to explore further. 

\textbf{Solving Rubik Table problems with fewer shuffles.} It is 
conceivable that a typical Rubik Table problem may not need $2n$ shuffles 
to solve. Given its general applicability, it is interesting to seek
algorithms for optimally solving Rubik Table problems minimizing the 
total number of shuffles. Another directly related question to ponder 
here is the computational complexity of optimizing the number of 
shuffles; the problem appears to be intractable. The same questions can 
be asked on the labeled Rubik Table problems and other generalizations. 

\textbf{Bound gap.}
Whereas we know that it is not possible to reach $O(nd)$ for \lsr for 
arbitrary $n$ and $d$, we do not know whether the same is true for \usr. 
In our algorithmic solution, though we achieve $O(nd)$ for arbitrarily 
large but fixed $\frac{\log d}{\log n}$, we have not fully closed the gap for \usr. 
In the approach that we have used, the issue is caused by the $3k$ 
recursive calls. The $3$ there is where the $3^m$ factor (in the $O(3^mnd)$ 
complexity stated in Theorem~\ref{t:lsr-d-ge-n}) comes from. For \usr, 
we were able to further drop the required number of moves to $O(2^mnd)$
with a much more involved argument than repeatedly applying Rubik Table 
results. 
Reducing the number of recursive calls may get us closer to closing 
the small remaining gap between the lower and upper bounds.

\textbf{Hardness of stack rearrangement}. The question of whether 
polynomial time algorithms 
can be designed for optimally solving \usr and \lsr remains open. We 
conjecture that both \usr and \lsr are NP-hard. In this regard, it 
may be interesting to study the case of constant $d$. Whereas the 
case of $d = 1$ can be readily solved, larger $d$ appears to be 
challenging. 

\textbf{Utility of multiple buffer stacks} In the current study, we 
have mainly examined the case of using a single buffer stack. We also 
show that using $\sqrt{n}$ empty buffer stacks allow the resolution 
of \usr in $O(nd)$ pop-pushes without any additional conditions imposed 
between $n$ and $d$ as required by, e.g., Theorem~\ref{t:lsr-d-cn}. 
A natural question to ask is for what values 
of $b \in [1, \sqrt{n})$, $b$ empty buffer stacks would enable solving 
\usr in $O(nd)$ pop-pushes. As have been discussed, it is not clear that 
$\sqrt{n}$ buffer stacks are sufficient for solving \lsr in $O(nd)$ 
pop-pushes for arbitrary $n$ and $d$, which also warrants further examination. 

\textbf{Other queuing models} As generalizations to the current problem,
it could be interesting to study a two-dimensional stack setting, e.g., 
items may be accessed both from the top or from the left side. Does such 
a setting, which provides similar storage capacity as stacks, allow 
more access flexibility? One may also replace a stack with a queue that 
may be accessed from both ends. Many additional settings similar to 
these two could be examined. 

\section*{Acknowledgment}
The work is supported in part by NSF awards IIS-1734419, IIS-1845888 
and CCF-1934924.

\bibliographystyle{IEEETran}
\bibliography{bib/bib}
\end{document}